\documentclass[a4paper,10pt,twoside]{article}

\usepackage{dsfont}
\usepackage{microtype}
\usepackage{times}  
\usepackage{geometry}
\usepackage{macros}
\usepackage{authblk}

\begin{document}

\title{Concentration bounds for CVaR estimation:\\[0.5ex] The cases of light-tailed and heavy-tailed distributions} % \thanks is optional. Insert line breaks with \\

\author[1]{Prashanth L. A.}
\affil[1]{\small  Department of Computer Science and Engineering,
  Indian Institute of Technology Madras}
\author[2]{Krishna Jagannathan}
\affil[2]{Department of Electrical Engineering \\ Indian Institute of Technology Madras}
\author[3]{Ravi Kumar Kolla}
\affil[3]{ABInBev, Bangalore}
\date{}
\maketitle

%%%%%%%%%%%%%%%%%%%%%%%%%%%%%%%%%%%%%%%%%%%%%%%%%%%%%%%%%%%%%%%%%%%
%%                                                               %%
%% abstract:                                                     %%
%%                                                               %%
%%%%%%%%%%%%%%%%%%%%%%%%%%%%%%%%%%%%%%%%%%%%%%%%%%%%%%%%%%%%%%%%%%%

\begin{abstract}
Conditional Value-at-Risk (CVaR) is a widely used risk metric in applications such as finance. We derive concentration bounds for CVaR estimates, considering separately the cases of light-tailed and heavy-tailed distributions. In the light-tailed case, we use a classical CVaR estimator based on the empirical distribution constructed from the samples. For heavy-tailed random variables, we assume a mild `bounded moment' condition, and derive a concentration bound for a truncation-based estimator. Notably, our concentration bounds enjoy an exponential decay in the sample size, for heavy-tailed as well as light-tailed distributions. To demonstrate the applicability of our concentration results, we consider a CVaR optimization problem in a multi-armed bandit setting. Specifically, we address the best CVaR-arm identification problem under a fixed budget. We modify the well-known successive rejects algorithm to incorporate a CVaR-based criterion. Using the CVaR concentration result, we derive an upper-bound on the probability of incorrect identification by the proposed algorithm.   
\end{abstract}

%%%%%%%%%%%%%%%%%%%%%%%%%%%%%%%%%%%%%%%%%%%%%%%%%%%%%%%%%%%%%%%%%%%%%%%%%%%%%%%
%%%%%%%%%%%%%%%%%%%%%%%%%%%%%%%%%%%%%%%%%%%%%%%%%%%%%%%%%%%%%%%%%%%%%%%%%%%%%%%
%%%%%%%%%%%%%%%%%%%%%%%%%%%%%%%%%%%%%%%%%%%%%%%%%%%%%%%%%%%%%%%%%%%%%%%%%%%%%%%
%%%%%%%%%%%%%%%%%%%%%%%%%%%%%%%%%%%%%%%%%%%%%%%%%%%%%%%%%%%%%%%%%%%%%%%%%%%%%%%

\section{Introduction}
\label{sec:intro}
In applications such as portfolio optimization in finance,  the quality of a portfolio is not satisfactorily captured by the expected value of return. Indeed, in such applications, a more risk-sensitive metric is desirable, so as to capture typical losses in the case of adverse events.  Value-at-Risk (VaR) and Conditional-Value-at-Risk (CVaR) are two risk-aware metrics, which are widely used in applications such as portfolio optimization and insurance. VaR at level $\alpha \in (0, 1)$ conveys the maximum loss incurred by the portfolio with a confidence of $\alpha.$ In other words, the portfolio incurs a loss greater than VaR at level $\alpha$ with probability $1 - \alpha.$ In turn, CVaR at level $\alpha \in (0, 1)$ captures the expected loss incurred by the portfolio, \emph{given} that the losses exceed VaR at level $\alpha.$ CVaR has an advantage over VaR, in that the former is a \emph{coherent\footnote{A risk measure is said to be coherent, if it is monotonic, translation invariant, sub-additive, and positive homogeneous.}} risk measure~\cite{artzner1999coherent}. 

In this paper, we derive concentration bounds for CVaR estimators, for both light-tailed and heavy-tailed random variables. For light-tailed distributions, our concentration bound uses a classical CVaR estimator based on the empirical distribution.  For the heavy-tailed case, we employ a truncation-based CVaR estimator, and derive a concentration result under a mild assumption: the $p$th moment of the distribution is assumed to exist, for some $p>1.$  Notably, our concentration bounds enjoy an exponential decay in the sample size, for heavy-tailed as well as light-tailed distributions. Our results also subsume or strengthen existing CVaR concentration results, as we discuss in the next subsection.  We believe our bounds are order optimal, and the dependence the number of samples as well as the accuracy cannot be improved. 

In order to highlight an important application for our CVaR concentration results, we consider a stochastic bandit set-up with a risk-sensitive metric for measuring the quality of an arm. In particular, we consider a  $K$-armed stochastic bandit setting, and study the problem of finding the arm with the \emph{lowest CVaR value} (at a fixed level $\alpha \in  (0, 1)$) in a fixed budget setting.  We propose an algorithm for the best CVaR arm identification that is inspired by successive-rejects \cite{audibert2010best}. Using our CVaR concentration bound, we establish an upper bound on the probability of incorrect arm identification by our algorithm at the end of the given budget. 

\subsection{Related Work}
For the case of bounded distributions, a popular CVaR estimate has been shown to exponentially concentrate around the true CVaR -- see \cite{brown2007large,wang2010deviation}. In comparison to CVaR, obtaining a concentration result for VaR is easier, and does not require assumptions on the tail of the distribution -- see \cite{ravi2018cvar}, a paper which also derives a one-sided CVaR concentration bound. More recent work~\cite{thomas2019concentration} considers CVaR concentration for distributions with bounded support on one side. In another recent paper \cite{2019arXiv190210709B}, the authors derive an exponentially decaying concentration bound for the case of sub-Gaussian distributions, using a concentration result \cite{fournier2015rate} for the Wasserstein distance between the empirical and the true distributions. However, the above approach leads to poor concentration bounds (with power law decay in the sample size) for other relevant disribution classes, such as light-tailed and bounded-moment distributions.

 While bandit learning has a long history, dating back to \cite{thompson1933likelihood}, risk-based criteria have been considered only recently. \cite{sani2012risk} consider mean-variance optimization in a regret minimization framework. In the best arm identification setting, VaR-based criteria has been studied by \cite{david2016pure} and \cite{davidpac}. CVaR-based criteria has been explored in a bandit context by \cite{galichet2013exploration}, albeit with an assumption of bounded arms' distributions.

%We consider the problem of finding the arm with the lowest CVaR value at level $\alpha \in  (0, 1)$, called as the best CVaR arm, in $K$-armed stochastic bandits with fixed budget setting. Our contributions under this paradigm are as follows:
%We formulate the best CVaR arm identification problem with fixed budget under $K$-armed stochastic bandits with two distinct assumptions for the arms' distributions. In the first case, we assume the arms' distributions are light-tailed. In the second case, which is more general, we assume that the $p$-th moment of arms' distributions is bounded, with $p>1$.
%For both case, we derive novel concentration bound for CVaR estimation, using a well-known estimation scheme based on empirical distribution functions.
%We propose a successive-rejects (SR) type algorithm for the best CVaR arm identification, and then derive an upper bound on the probability of incorrect identification by the same algorithm at the end of the given budget. 

% Overall, we study the best CVaR arm identification problem with fixed budget and provide upper bounds on the probability of incorrect identification by our algorithm. 

The rest of this paper is organized as follows: Section~\ref{sec:model} presents the preliminaries. Sections~\ref{sec:est-light} and \ref{sec:est-heavy} present the key concentration bounds for light and heavy-tailed distributions, respectively. Section~\ref{sec:bandits} provides bandit algorithms and their analyses for the problem of the best CVaR arm identification with fixed budget under $K$-armed stochastic bandits. %Section~\ref{sec:cvar-numerical-results} presents the results from simulation experiments on illustrative bandit problem settings. 
The proofs are contained in Section~\ref{sec:proofs}, and Section~\ref{sec:conclusions} concludes the paper.

\section{Preliminaries}
\label{sec:model}
%In this section, we define VaR and CVaR at level $\alpha \in (0, 1)$ of a r.v. $X$. 
%Given a r.v. $X$ and a level $\alpha \in (0, 1),$ we use $v_\alpha(X)$ and $c_\alpha(X)$ to denote VaR and CVaR at level $\alpha$ of $X$ respectively 
Given a r.v. $X$ with cumulative distribution function (CDF) $F(\cdot)$, the VaR $v_\alpha(X)$ and CVaR $c_\alpha(X)$ at level $\alpha\in (0,1)$ are defined as follows~\footnote{For notational brevity, we omit $X$ from the notations $v_\alpha(X)$ and $c_\alpha(X)$  whenever the underlying the r.v. can be understood from the context.}:
\begin{align}
v_\alpha(X) & = \inf \lbrace \xi : \prob{X \leq \xi} \geq \alpha \rbrace, \textrm{ and ~} 
c_{\alpha}(X)   =  v_{\alpha}(X)  + \frac{1}{1 - \alpha} \mathbb{E} \left[ X - v_{\alpha}(X) \right] ^+ \label{eq:cvar-def},
\end{align}
where we have used the notation $[X]^+ = \max (0, X ).$ Typical values of $\alpha$ chosen in practice are $0.95$ and $0.99$. 
We make the following assumption for the purpose of CVaR estimation as well as for the concentration bounds derived later.
\\[0.5ex]
\textbf{(C1)} The r.v. $X$ is continuous with strictly increasing CDF.

Under (C1), $v_\alpha(X)$ is a solution to $\prob{X \leq \xi} = \alpha$, i.e., $v_\alpha(X) = F^{-1}(\alpha)$. Further, if $X$ has a positive density at $v_\alpha(X)$, then $c_\alpha(X) = \expect{X \vert X \geq v_\alpha(X)}$~(cf. \cite{sun2010asymptotic}).

\section{CVaR estimation: Light-tailed case}
\label{sec:est-light}
In this section, we define empirical CVaR, provide a concentration result for CVaR estimation assuming that the underlying distribution is light-tailed, and subsequently present a multi-armed bandit application.
\subsection{VaR and CVaR estimation}\label{sec:varcvar_est}
Let $\lbrace X_i \rbrace_{i=1}^n$ be $n$ i.i.d. samples  drawn from the distribution of $X$.  Let $\lbrace X_{[i]} \rbrace_{i=1}^n$ be the order statistics of $\lbrace X_i \rbrace_{i=1}^n$, i.e., $X_{[1]} \geq X_{[2]} \dots \geq X_{[n]}.$ Let $\hat{F}_n(\cdot)$ be the empirical distribution function calculated using $\lbrace X_i \rbrace_{i=1}^n$, defined as
$
\hat{F}_n (x) = \frac{1}{n} \sum_{i=1}^n \indic{X_i \leq x}, \forall
 x \in \mathbb{R}.
$
Notice that CVaR is a conditional expectation, where the conditioning event requires VaR. Thus, CVaR estimation requires VaR to be estimated as well.
Let $\hat{v}_{n, \alpha}$ and $\hat{c}_{n, \alpha}$ denote the estimates of VaR and CVaR at level $\alpha$ using the $n$ samples above. These quantities are defined as follows \cite{serfling2009approximation}: 
\begin{align}
\hat{v}_{n, \alpha} & = X_{\left[  \lfloor n(1-\alpha) \rfloor \right]}, \textrm{ and ~}
\hat{c}_{n, \alpha}  = \frac{1}{n(1-\alpha)} \sum_{i=1}^n X_i \indic{X_i \geq \hat{v}_{n, \alpha}}. \label{eq:cvar-estimate}
\end{align} 
 
\subsection{Concentration bounds}
In the case of  distributions with bounded support, a concentration result for  CVaR  exists in the literature~\cite{wang2010deviation}.  For the case of unbounded distributions, deriving a CVaR concentration result becomes considerably easier when the form of distributions are known, \emph{i.e.,} when the closed-form expressions of VaR and CVaR can be derived. To illustrate, consider the case of a Gaussian r.v. $X$ with mean $\mu$ and variance $\sigma^2$. Let $Q \left( \xi \right) = \frac{1}{\sqrt{2 \pi}} \int_\xi^\infty \exp \left(-x^2/2 \right) dx$. Notice that $Q(-x) = 1 - Q(x)$ and also that $F_X (\xi) = Q \left( \frac{\mu - \xi}{\sigma} \right)$. Hence, $v_\alpha(X)$ is the solution to $Q \left( \frac{\mu - \xi}{\sigma} \right) = \alpha$, which implies that
\begin{align}
\label{eq:var-expression-for-gaussian}
v_\alpha(X) = \mu - \sigma Q^{-1}\left( \alpha \right). 
\end{align}
The CVaR $c_\alpha(X)$ for Gaussian $X$ can be shown,  using Acerbi's formula \cite[pp. 329]{chatterjee2014practical}, to be equal to $\mu \left( \frac{\alpha}{1-\alpha} \right) + \sigma c_\alpha(Z)$, 
where $Z$ is the standard Gaussian random variable \emph{i.e.,} $Z \sim \mathcal{N}(0,1).$ %(See Appendix \ref{sec:appendix-cvar-calc} for a proof).

It is clear from the above argument that estimates of $\mu$ and $\sigma$ are sufficient to estimate $c_\alpha(X)$ for the Gaussian case. Sample mean $\hat\mu_n$ and sample variance $\hat\sigma_n^2$ (computed using $n$ samples from the distribution of $X$) would serve this purpose and we obtain $\hat c_n = \hat\mu \left( \frac{\alpha}{1-\alpha} \right) + \hat\sigma c_\alpha(Z)$ as a proxy for $c_\alpha(X)$. Given standard concentration bounds for these quantities through Hoeffding and Bernstein's inequalities, it is straightforward  to establish that $\hat{c}_{n, \alpha}$ concentrates exponentially around $c_\alpha(X).$ Similarly, for the case of exponential random variables, we can exploit the memoryless property to derive an explicit expression for CVaR, in terms of the mean $\mu$ and the level $\alpha.$ 

We therefore focus on distributions that do not have closed-form expressions for VaR and CVaR.  In such a setting, the CVaR has to be estimated directly from the available samples.  
However, for establishing concentration bounds for the CVaR, which involves conditioning on a tail event, it is common to make some assumptions on the tail distribution. In \cite{2019arXiv190210709B}, an exponentially decaying CVaR concentration result is derived for the class of sub-Gaussian random variables, using a Wasserstein distance approach. However, the same approach provides unsatisfactory results (with power-law decay)  for light-tailed as well as heavy-tailed distributions with bounded higher moments.

We now define the class of light-tailed distributions , while heavy-tailed distributions are handled in the next section.
\begin{definition}\label{def:subexp}
A r.v. $X$ is said to be light-tailed if there exists a $c_0>0$ such that $\mathbb{E} [\exp (\lambda X)]<\infty$ for all $|\lambda|<c_0.$ 
\end{definition}
The following lemma provides equivalent characterizations of light-tailed distributions -- see \cite[Theorem 2.2]{wainwright2019high}.
\begin{lemma}
The following statements are equivalent:
\begin{enumerate}
\item $X$ is light-tailed.
\item There exist constants $\eta_1,\eta_2>0$ such that 
$\prob{|X| \geq t} \leq \eta_1 \exp(-\eta_2 t),\quad \forall t > 0.$
\item There exist non-negative parameters $\sigma$ and $b$ such that
\begin{align}
\E\left[\exp\left(\lambda X\right)\right] \le \exp\left(\dfrac{\lambda^2 \sigma^2}{2}\right), \text{ for any } |\lambda| < \frac{1}{b}.
\label{eq:subexp-equiv}
\end{align}
\end{enumerate}
\end{lemma}
The following result presents a concentration bound for the case of light-tailed distributions:
\begin{theorem}[\textbf{\textit{CVaR concentration: Light-tailed case}}]
\label{thm:cvar-concentration-sub-exp}
Let $\lbrace X_i \rbrace_{i=1}^n$ be a sequence of i.i.d. r.v.s. Assume (C1). Let $\hat{c}_{n, \alpha}$ be the CVaR estimate given in~\eqref{eq:cvar-estimate} formed using the above set of samples. 
Suppose that $X_i$, $i=1,\ldots,n$ are light-tailed with parameters $\sigma,b$, and VaR $v_\alpha$. Then, for any $\epsilon > 0$, we have
\begin{align*}
&\prob{ \left| 
\hat c_{n,\alpha} - c_\alpha \right| > \epsilon} \le 
\left\{\begin{array}{c}6\exp\left[-\frac{cn\epsilon^2(1-\alpha)^2}{2(\sigma^2+v_\alpha^2)}\right],\ 0\leq\epsilon \leq \frac{\sigma^2+v_\alpha^2}{b(1-\alpha)},\\
2\exp\left[-\frac{n\epsilon(1-\alpha)}{4b}\right]\!+\!6 \exp\left[-cn\epsilon^2(1-\alpha)^2\right],\ \epsilon\! >\! \frac{\sigma^2+v_\alpha^2}{b(1-\alpha)}, \end{array}\right.
\end{align*} 
where $c$ is a distribution dependent constant.
\end{theorem}
A few remarks concerning the result above are in order.
\begin{remark}
The bound in the theorem above is significantly better than the two-sided bound obtained in \cite{2019arXiv190210709B} for the light-tailed case. In particular, the bound in the theorem above has an exponential tail decay irrespective of whether $\epsilon$ is large or small, while the bound in \cite{2019arXiv190210709B} has an exponential decay for small $\epsilon$, and a power law for large $\epsilon$. For a light-tailed r.v., one expects a tail behavior similar to that of Gaussian with constant variance for small $\epsilon$, and an exponential decay for large $\epsilon$, and our bound is consistent with this expected behavior. 
\end{remark}
\begin{remark}
In comparison to the one-sided bound for light-tailed r.v.s, obtained in \cite{ravi2018cvar}, our bound exhibits much better dependence w.r.t.  the number of samples $n$ as well as the accuracy $\epsilon$. More importantly, since our bound is two-sided, it opens avenues for a bandit application, while a one-sided bound is insufficient for this purpose.
\end{remark}
%\begin{remark}
%The authors in \cite{2019arXiv190210709B} use a Wasserstein distance-based approach to establish CVaR concentration, while we employ a quantile-based approach. The advantage with the former approach is that the constants in their bounds depend only on the parameters $\sigma,b$ parameters (see \eqref{eq:subexp-equiv}). On the other hand, the constants in the bounds derived using our approach require distributional information through the value of density in the neighborhood of the true VaR. In the absence of such information, the Wasserstein distance-based bounds are preferable, while in situation where distributional information is available, our bounds are better since the constants are optimized, while constants in Wasserstein distance-based bounds are on the conservative side.   
%\end{remark}

%%%%%%%%%%%%%%%%%%%%%%%%%%%%%%%%%%%%%%%%%%%%%%%%%%%%%%%
% end of light-tailed case
%%%%%%%%%%%%%%%%%%%%%%%%%%%%%%%%%%%%%%%%%%%%%%%%%%%%%%%

In the following section, we provide a multi-armed bandit algorithm that incorporates a CVaR objective, and analyze the finite-time performance of this algorithm using the bound derived in Theorem \ref{thm:cvar-concentration-sub-exp}.
%%%%%%%%%%%%%%%%%%%%%%%%%%%%%%%%%%%%%%%%%%%%%%%%%%%%%%%%%%%%%%%%%%%%%%%%%%%%%%%%
%%%%%%%%%%%%%%%%%%%%%%%%%%%%%%%%%%%%%%%%%%%%%%%%%%%%%%%%%%%%%%%%%%%%%%%%%%%%%%%%
%%%%%%%%%%%%%%%%%%%%%%%%%%%%%%%%%%%%%%%%%%%%%%%%%%%%%%%%%%%%%%%%%%%%%%%%%%%%%%%%
%%%%%%%%%%%%%%%%%%%%%%%%%%%%%%%%%%%%%%%%%%%%%%%%%%%%%%%%%%%%%%%%%%%%%%%%%%%%%%%%
\subsection{Application: Multi-armed bandits}
\label{sec:bandits}
We consider a $K$-armed stochastic bandit problem, with arms' distributions $\mathcal{P}_1,\ldots, \mathcal{P}_K$. We study the problem of finding the arm with the \emph{lowest CVaR value} (at a fixed level $\alpha \in  (0, 1)$) in a fixed budget setting. In this setting, a bandit algorithm interacts with the environment over a given budget of $n$ rounds. In each round $t=1,\ldots,n$, the algorithm pulls an arm $I_t \in \{1,\ldots,K\}$ and observes a sample cost from the distribution $\mathcal{P}_{I_t}$. At the end of the budget $n$ rounds, the bandit algorithm recommends an arm $J_n$ and is judged based on the probability of incorrect identification, i.e., $\prob{J_n\ne i^*}$, where $i^*$ denotes the best arm. Earlier works use the expected value to define the best arm, while we use CVaR. 

Let $c^i_\alpha$ and $v^i_\alpha$ denote the CVaR and VaR of the arm $i$ at level $\alpha.$ Let $c^* = \min_{i=1,\ldots,K} c^i_\alpha,$ and $i^*$ be the arm that achieves this minimum. The goal is to devise an algorithm for which $\prob{J_n \ne i^*}$ is small after $n$ rounds of sampling. Let arm-$[i]$ denotes the $i^{th}$ lowest CVaR valued arm. Let $\Delta_i = c^i_\alpha - c^{i^*}_\alpha$ denote the gap between the CVaR values of arm-$i$ and the optimal arm. %Let $\Delta_L$ denote a lower bound on the gaps \emph{i.e.,} $\Delta_L \le \Delta_i$ for all $1 \le i \le K.$ 

\begin{algorithm*}
\begin{algorithmic}
%\vspace{1ex}
%\State \textbf{Input:}  Budget $n$, gap lower bound $\Delta_L$.  
\State \textbf{Initialization:}   
Set    $A_1 = \{ 1,\ldots,K\},$  $\overline{\log}K  = \frac{1}{2} + \sum \limits_{i=2}^{K} \frac{1}{i}, n_0 = 0$, $ n_k = \left\lceil \frac{1}{\overline{\log}K} \frac{n-K}{K+1-k} \right\rceil$, $k= 1,\ldots,K-1.$
    \vspace{1ex}
\For{$k = 1,2,\ldots,K-1$}
\vspace{1ex}	
	\State Play each arm in $A_k$ for $(n_k - n_{k-1})$ times. 
	\vspace{1ex}
	\State  Compute the CVaR estimate $\hat c^i_{\alpha, n_k}$ for each arm $i \in A_k$ using~\eqref{eq:cvar-estimate}.
	\vspace{1ex}
	\State Set $A_{k+1} = A_k \setminus \underset{i \in A_k}{\arg \max}\ \hat{c}^i_{\alpha, n_k}$, i.e., remove the arm with the highest empirical CVaR, with ties broken arbitrarily.
	\vspace{1ex}
\EndFor
\vspace{1ex}
\State {\bf Output:} Return the solitary element in $A_K$.
\vspace{1ex}
\end{algorithmic}
\caption{CVaR-SR algorithm}
\label{alg:1spsa}
\end{algorithm*}

Algorithm~\ref{alg:1spsa} presents the pseudo code of our CVaR-SR algorithm, designed to find the CVaR-optimal arm under a fixed budget. The algorithm is a variation of the regular successive rejects (SR) algorithm \cite{audibert2010best}, with the following key difference:  regular SR uses sample mean to estimate the expected value of each arm, while CVaR-SR used empirical CVaR, as defined in \eqref{eq:cvar-estimate}, to estimate CVaR for each arm. The elimination logic, i.e., having $K-1$ phases, and removing the worst arm (according to sample estimates of CVaR) at the end of each phase, is borrowed from regular SR.  
%\begin{remark}
%From a practical standpoint, the best arm identification approach is quite appropriate in the context of CVaR optimization. Consider an application such as financial portfolio optimization. The goal here would be to find an investment strategy such that the losses incurred in the worst-case scenario (e.g., stock market crash) are minimized. Such a problem can be modeled well with a CVaR objective. However, estimating CVaR in a real-life market scenario is challenging, while one could model the various processes in the finanical application considered, and build a simulator. The problem of finding the CVaR-optimal investment strategy can be addressed by using the simulator -- an approach that falls under the realm of simulation optimization \cite{fu2015handbook}. \end{remark}

In the following result, we analyze the performance of CVaR-SR algorithm for light-tailed distributions.
\begin{theorem}[\textbf{\textit{Probability of incorrect identification}}]
\label{thm:cvar-sr-sub-gaussian}
Consider a $K$-armed stochastic bandit, where the arms' distributions satisfy (C1) and are light-tailed. For a given budget $n$, the arm, say $J_{n}$, returned by the CVaR-SR algorithm satisfies:
\begin{align*}
\prob{J_{n} \neq i^*} \leq 4K(K-1) \exp \left( -\frac{(n-K)(1-\alpha) G_{\max}}{{H\overline{\log}K}}  \right),
\end{align*} 
 where $G_{\max}$ is a problem dependent constant that does not depend on the underlying CVaR gaps and $n$, and
 \[H = \max_{i \in \lbrace 1, 2 \dots, K \rbrace} \frac{i}{\min \lbrace \Delta_{[i]}/2, \Delta_{[i]}^2/4 \rbrace }.\] 
\end{theorem}

%%%%%%%%%%%%%%%%%%%%%%%%%%%%%%%%%%%%%%%%%%%%%%%%%%%%%%%%%%%%%%%%%%%%%%%%%%%%%%%%
%%%%%%%%%%%%%%%%%%%%%%%%%%%%%%%%%%%%%%%%%%%%%%%%%%%%%%%%%%%%%%%%%%%%%%%%%%%%%%%%
%%%%%%%%%%%%%%%%%%%%%%%%%%%%%%%%%%%%%%%%%%%%%%%%%%%%%%%%%%%%%%%%%%%%%%%%%%%%%%%%
%%%%%%%%%%%%%%%%%%%%%%%%%%%%%%%%%%%%%%%%%%%%%%%%%%%%%%%%%%%%%%%%%%%%%%%%%%%%%%%%
\section{CVaR estimation: Heavy-tailed case}
\label{sec:est-heavy}
As mentioned before, an alternative proof approach using Wasserstein distance~\cite{2019arXiv190210709B} provides weak concentration rates for distributions with bounded higher moments - a gap that we address in this work.
In particular, we employ a truncation-based estimator for CVaR to handle the case when the underlying distribution satisfies the following assumption:
\\[0.5ex]
\textbf{(C2)} $\exists p \in (1,2], u$ such that $\E[ |X|^p] < u < \infty$.

\subsection{CVaR estimation}
Recall that $\lbrace X_{[i]} \rbrace_{i=1}^n$ denote the order statistics of $n$ i.i.d. samples  drawn from the distribution of $X$. Using the VaR estimate $\hat{v}_{n, \alpha}$, as defined earlier in Section~~\ref{sec:varcvar_est}, we propose a truncation-based estimator $\hat{c}_{n, \alpha}$ for CVaR at level $\alpha$, defined as follows:
\begin{align}
\hat{c}_{n, \alpha} = \frac{1}{n(1-\alpha)} \sum_{i=1}^n X_i\indic{\hat{v}_{n, \alpha} \leq X_i \leq B_i}, \textrm{ where } B_i =  \left(\frac{u i}{\log(1/\delta)}\right)^{1/p}.\label{eq:cvar-estimate-bdd}
\end{align} 
   In~\eqref{eq:cvar-estimate}, $B_i$ represents a truncation level of $X_i$, and the choice for $B_i$ given above is under the assumption that $\E[ |X|^p] < u < \infty$ for some $p \in (1,2]$.
 Such a truncation based estimator has been employed in the context of expected regret minimization with heavy-tailed random variables in \cite{bubeck2013bandits}. Intuitively, the truncation level serves to discard very large samples values early on, as $B_i$ is  set to grow slowly with $i.$
\subsection{Concentration bounds}
In particular, the following result is more general, as it can handle heavy-tailed distributions that satisfy (C2). 
%We consider distributions with bounded $p$th moment, i.e., $\E [|X|^p]<u<\infty,$ where $p>1.$ 
\begin{theorem}[\textbf{\textit{CVaR concentration: Bounded moment case}}]
\label{thm:cvar-concentration-bounded-moment}
Let $\lbrace X_i \rbrace_{i=1}^n$ be a sequence of i.i.d. r.v.s satisfying (C1) and (C2). 
%Assume that the random variables have a bounded moment, i.e., $\exists p \in (1,2], u$ such that $\E[ |X|^p] < u < \infty$.
 Let $\hat{c}_{n, \alpha}$ be the CVaR estimate given in~\eqref{eq:cvar-estimate} formed using the above set of samples. Fix $\epsilon>0$.
 \begin{itemize} 
\item[(i)] For the case when $p \in (1,2),$,
\begin{align*}
\prob{ \left| 
\hat c_{n,\alpha} - c_\alpha \right| > \epsilon} \le 8\exp\left(-  cn (1-\alpha)^{\frac{p}{(p-1)}} \epsilon^{\frac{p}{(p-1)}}\right),
\end{align*} 
where $c$ is a distribution-dependent constant.
\item[(ii)] For the case when the distribution of $X$ has a bounded second moment, i.e., $p=2$, 
\begin{align*}
\prob{\left|\hat{c}_{n, \alpha} - c_{\alpha}\right|   > \epsilon}  &\le 8\exp\left(-  c'n (1-\alpha)^2 \epsilon^2\right),
\end{align*}
where $c'$ is a distribution-dependent constant.
 \end{itemize}
\end{theorem}

\begin{remark}
A bandit application for the case of heavy-tailed distributions can be worked out using arguments similar to that in Section \ref{sec:bandits}. The main difference is that the SR algorithm in the heavy-tailed case would involve a truncated estimator, and a slightly different hardness measure that is derived using Theorem \ref{thm:cvar-concentration-bounded-moment}. We omit the details due to space constraints.
\end{remark}
%%%%%%%%%%%%%%%%%%%%%%%%%%%%%%%%%%%%%%%%%%%%%%%%%%%%%%%%%%%%%%%%%%%%%%%%%%%%%%%
%%%%%%%%%%%%%%%%%%%%%%%%%%%%%%%%%%%%%%%%%%%%%%%%%%%%%%%%%%%%%%%%%%%%%%%%%%%%%%%
%%%%%%%%%%%%%%%%%%%%%%%%%%%%%%%%%%%%%%%%%%%%%%%%%%%%%%%%%%%%%%%%%%%%%%%%%%%%%%%
\section{Proofs}
\label{sec:proofs}
\subsection{Proof of Theorem \ref{thm:cvar-concentration-sub-exp}}
Before providing the main proof, we note that 
empirical CVaR, as defined in \eqref{eq:cvar-estimate}, involves empirical VaR, and it is natural to expect that empirical CVaR concentration would require empirical VaR to concentrate as well. VaR concentration bounds have been derived recently in \cite{ravi2018cvar}, and we recall their result below. This result will be used to establish the bound in Theorem \ref{thm:cvar-concentration-sub-exp}. 
\begin{lemma}[\textbf{\textit{VaR concentration}}]
\label{prop:var-concentration-bound1}
Suppose that (C1) holds. For any $\epsilon > 0,$ we have
\begin{align*}
\prob{\vert \hat{v}_{n, \alpha} - v_\alpha \vert \geq \epsilon} \leq 2 \exp \left(  -2nc\epsilon^2   \right), 
\end{align*}
where $c$ is a constant that depends on the value of the density $f$ of the r.v. $X$ in a neighbourhood of $v_\alpha(X).$
\end{lemma}

\begin{proof}[Proof of Theorem \ref{thm:cvar-concentration-sub-exp}]

Notice that
 \begin{align}
\hat{c}_{n, \alpha} &= \hat{v}_{n,\alpha} +\frac{1}{n(1-\alpha)} \sum_{i=1}^n \left( X_i - \hat{v}_{n, \alpha} \right)\indic{\hat v_{n,\alpha} \le X_i} \nonumber\\
& = v_\alpha + \frac{1}{n(1-\alpha)} \sum_{i=1}^n \left( X_i - v_\alpha \right)\indic{v_\alpha \le X_i} + \e_n, \label{eq:asd1} 
\end{align}
where 
%\begin{align*}
%&\e_n = \left( \hat{v}_{n, \alpha} - v_\alpha \right) + \\
%&\frac{1}{n(1-\alpha)} \sum_{i=1}^n  \left( X_i - \hat{v}_{n, \alpha} \right)\left[ \indic{\hat{v}_{n, \alpha} \le X_i } - \indic{v_\alpha \le X_i } \right]\\
%&= \left( \hat{v}_{n, \alpha} - v_\alpha \right) 
%+ \frac{1}{n(1-\alpha)} \sum_{i=1}^n  \left( v_{\alpha} - \hat{v}_{n, \alpha} \right)\indic{\hat{v}_{n, \alpha} \le X_i }  \\
%&+ \frac{1}{n(1-\alpha)} \sum_{i=1}^n\left( X_i - v_{\alpha} \right)\left[ \indic{\hat{v}_{n, \alpha} \le X_i } - \indic{v_\alpha \le X_i } \right] \\
%&= \left( \hat{v}_{n, \alpha} - v_\alpha \right) 
%+ \frac{\left( v_\alpha- \hat{v}_{n, \alpha}   \right)}{(1-\alpha)} \left( \hat F_n(B_i) - \hat F_n( \hat{v}_{n, \alpha}) \right)\\
%&+ \frac{1}{n(1-\alpha)} \sum_{i=1}^n\left( X_i - v_{\alpha} \right)\left[ \indic{\hat{v}_{n, \alpha} \le X_i } - \indic{v_\alpha \le X_i } \right]. 
%\end{align*}
\begin{align*}
e_n &= \frac{\hat{v}_{n, \alpha} - v_\alpha}{1-\alpha} \left[ \hat{F}_n(\hat{v}_{n, \alpha}) - \alpha\right] + \frac{1}{n} \sum_{i=1}^n \frac{X_i - v_\alpha}{1-\alpha} \left[ \indic{X_i \geq \hat{v}_{n, \alpha} } - \indic{X_i \geq v_\alpha } \right].
\end{align*}
The reader is referred to the initial passage in the proof of Proposition 5 in \cite{ravi2018cvar} for a justification of the equality in \eqref{eq:asd1}.

Thus,
\begin{align*}
\left|\e_n\right| &\le
\frac{\vert v_\alpha - \hat{v}_{n, \alpha} \vert}{1-\alpha} \vert \alpha - \hat{F}_n (\hat{v}_{n, \alpha}) \vert  + \frac{\vert v_\alpha - \hat{v}_{n, \alpha} \vert}{1-\alpha} \vert \hat{F}_n(v_\alpha) - \hat{F}_n (\hat{v}_{n, \alpha}) \vert\\
& \leq \frac{\vert v_\alpha - \hat{v}_{n, \alpha} \vert}{1-\alpha} \Big[ 2 \vert \hat{F}_n (\hat{v}_{n, \alpha}) - F(v_\alpha) \vert 
+ \vert \hat{F}_n(v_\alpha) - F(v_\alpha) \vert \Big].\stepcounter{equation}\tag{\theequation}\label{eq:notsosimple}
\end{align*}
Using $\vert \hat{F}_{n} (\hat{v}_{n,\alpha}) - F(v_\alpha) \vert \leq 1/n$, we obtain
\begin{align*}
 \prob{\e_n > \epsilon} &\leq   \prob{ \frac{2}{n} \frac{1}{(1-\alpha)} \vert \hat{v}_{n,\alpha} - v_\alpha \vert > \frac{\epsilon}{2} } + \prob{\frac{1}{1-\alpha} \vert \hat{v}_{n,\alpha} - v_\alpha \vert \vert \hat{F}_n (v_\alpha) - F(v_\alpha) \vert > \frac{\epsilon}{2}}\\
&\le 2 \exp \left(  -nc_1 (1-\alpha)^2\epsilon^2   \right) + 2 \exp \left(  -n(1-\alpha)^2c_2\epsilon^2   \right),
\end{align*}
where the final inequality uses the concentration result in Lemma \ref{prop:var-concentration-bound1} to obtain the first term, while the second term can be arrived at as follows: Letting $\epsilon' = \frac{(1-\alpha)\epsilon}{2}$,
\begin{align}
&\prob{\vert \hat{v}_{n,\alpha} - v_\alpha \vert \vert \hat{F}_n (v_\alpha) - F(v_\alpha) \vert > \epsilon'} \le   \prob{\vert  \hat{v}_{n,\alpha} - v_\alpha \vert  > \frac{\epsilon'}{2}} \le  2 \exp \left(  -\frac{nc'\epsilon'^2}{4}   \right),\label{eq:term2bd}
\end{align}
%%%%%%%% Old stuff
%\begin{align*}
%&\prob{\vert \hat{v}_{n,\alpha} - v_\alpha \vert \vert \hat{F}_n (v_\alpha) - F(v_\alpha) \vert > \epsilon'} \\
%&\le  \prob{\vert v_{\alpha} (\hat{F}_n (v_\alpha) - F(v_\alpha)) \vert > \frac{\epsilon'}{2}} + \prob{\vert \hat{F}_n (v_\alpha) ( \hat{v}_{n,\alpha} - v_\alpha) \vert  > \frac{\epsilon'}{2}}\\
% & \le 2 \exp \left(  -\frac{n\epsilon'^2}{2v_{\alpha}^2}   \right) + 2 \exp \left(  -\frac{nc'\epsilon'^2}{4}   \right) \le 4 \exp \left(  -nc''\epsilon^2   \right),
%\end{align*}
%%% ends here
where the first inequality follows by using the fact that $\vert \hat{F}_n (v_\alpha) - F(v_\alpha) \vert \le 2$, since the empirical/true distributions are bounded above by $1$. The final inequality above uses the VaR concentration result from Lemma \ref{prop:var-concentration-bound1}.
Thus,
\begin{align}
 \prob{\e_n > \epsilon} &\leq   4 \exp \left(  -n(1-\alpha)^2c_3\epsilon^2   \right),\label{eq:en-bd}
\end{align}
for a distribution dependent constant $c_3.$

Next, using \eqref{eq:asd1}, the estimation error $\hat{c}_{n, \alpha} - c_{\alpha}$ can be written as
\begin{align*}
&\hat{c}_{n, \alpha} - c_{\alpha} = I_1 + e_n, \quad \textrm{where } I_1 = \frac{1}{1 - \alpha} \left[ \frac{1}{n} \sum_{i=1}^n \left( X_i - v_{\alpha} \right)^+ - \E \left[ \left( X - v_{\alpha}\right)^+\right]  \right].
\end{align*}
%\todo[inline]{Need to handle $v_\alpha$ term in RHS above}
For bounding the $I_1$ term on the RHS above, we use the fact that $\left(X-v_{\alpha}\right)^+$ is a light-tailed r.v. This can be argued as follows: Letting $\mu_\alpha^+ = \E \left[ \left( X - v_{\alpha}\right)^+\right]$,
\begin{align*}
&\prob{ \left( X_i - v_{\alpha} \right)^+ - \mu_\alpha^+ > \epsilon} =  \prob{X > v_{\alpha} + \mu_\alpha^+ +  \epsilon} \le c_1 \exp\left(-c_2 \left(v_\alpha + \epsilon\right)\right) \le c_1 \exp\left(-c_4 \epsilon\right),
\end{align*}
where $c_1, c_2,$ and $c_4$ are distribution-dependent constants. 
Next, using the fact that $X$ is light-tailed, we have
\begin{align*}
\E\left[ \exp\left[\lambda\left(\left( X - v_{\alpha} \right)^+ - \mu_\alpha^+\right)\right]  \right] \le  1 + \frac{\lambda^2 \E X^2}{2} + \frac{\lambda^2 v_\alpha^2}{2} + o\left(\lambda^2\right). 
\end{align*}
In the above, we have used the fact that $\E\left[\left( X - v_{\alpha} \right)^2 \indic{X\ge v_\alpha} \right] \le \E X^2 + v_\alpha^2$. 
Comparing with the following identity: 
\[\exp\left(\frac{\lambda^2\sigma^2}{2}\right) = 1 + \frac{\lambda^2 \sigma^2}{2} + \frac{\lambda^2 v_\alpha^2}{2} + o\left(\lambda^2\right),\]
  it is easy to see that $\left( X - v_{\alpha} \right)^+$ is a light-tailed r.v. with parameters $(\sigma^2+v_\alpha^2,b)$, whenever $X$ is light-tailed with parameters $(\sigma^2,b)$ (see \eqref{eq:subexp-equiv}). 

Using a standard light-tailed concentration result (cf. Theorem 2.2. in \cite{wainwright2019high}), we obtain
\begin{align}
&\prob{\left|I_1\right| > \epsilon} \le 
\left\{\begin{array}{c}2\exp{\left(-\frac{n\epsilon^2(1-\alpha)^2}{2(\sigma^2+v_\alpha^2)}\right)},\ 0\leq\epsilon \leq \frac{\sigma^2+v_\alpha^2}{b(1-\alpha)},\\
2\exp\left(-\frac{n\epsilon(1-\alpha)}{2b}\right),\ \epsilon > \frac{\sigma^2+v_\alpha^2}{b(1-\alpha)}, \end{array}\right.
\label{eq:I1-bd}
\end{align}
The main claim follows by using 
\begin{align*}
\prob{ \left| 
\hat c_{n,\alpha} - c_\alpha \right| > \epsilon} \le  \prob{ \left| 
I_1 \right| > \frac{\epsilon}{2}} +  \prob{\e_n > \frac{\epsilon}{2}},
\end{align*}
and substituting the bounds obtained in \eqref{eq:en-bd} and \eqref{eq:I1-bd} in the RHS above.
\end{proof}

%%%%%%%%%%%%%%%%%%%%%%%%%%%%%%%%%%%%%%%%%%%%%%%%%%%%%%%%%%%%%%%%%%%
% proof-bandit
%%%%%%%%%%%%%%%%%%%%%%%%%%%%%%%%%%%%%%%%%%%%%%%%%%%%%%%%%%%%%%%%%%%
\subsection{Proof of Theorem~\ref{thm:cvar-sr-sub-gaussian}}
\begin{proof}
We begin the proof by rewriting the CVaR concentration bound present in Theorem~\ref{thm:cvar-concentration-sub-exp} in a simplified manner as follows:
\begin{align}
\label{eq:cvar-concentration-sub-exp-simplified}
\mathbb{P} \left[ \vert \hat{c}_{n, \alpha} - c_\alpha \vert > \epsilon \right] \leq 8 \exp \left[ -n (1-\alpha)\min \lbrace \epsilon, \epsilon^2 \rbrace G \right],	
\end{align} 
where $G = \min \lbrace \frac{c(1-\alpha)}{2(\sigma^2 + v^2_\alpha)}, \frac{1}{4b}, c(1-\alpha) \rbrace.$\\[1ex]
Note that, if the CVaR-SR algorithm has eliminated the optimal arm in phase $i$ then it implies that at least one of the last $i$ worst arms \emph{i.e.,} one of the arms in $\lbrace [K], [K-1], \cdots, [K-i+1] \rbrace$ must not have been eliminated in phase~$i.$ Hence, we obtain
\begin{align}
& \prob{J_{n} \neq i^*} \leq \sum_{k=1}^{K-1} \sum_{i=K+1-k}^K \prob{\hat{c}^{i^*}_{{n_k}, \alpha} \geq \hat{c}^{[i]}_{{n_k}, \alpha}} \nonumber\\
& = \sum_{k=1}^{K-1}\!\sum_{i=K+1-k}^K\prob{\hat{c}^{i^*}_{{n_k}, \!\alpha} - c^{i^*}_\alpha - \hat{c}^{[i]}_{{n_k}, \alpha} +  c^{[i]}_\alpha  \geq c^{[i]}_\alpha - c^{i^*}_\alpha} \nonumber\\
& \leq \sum_{k=1}^{K-1} \sum_{i=K+1-k}^K \prob{\hat{c}^{i^*}_{{n_k }, \alpha} - c^{i^*}_\alpha \geq \frac{\Delta_{[i]}}{2}} +  
 \sum_{k=1}^{K-1} \sum_{i=K+1-k}^K \prob{c^{[i]}_\alpha - \hat{c}^{[i]}_{{n_k}, \alpha} \geq \frac{\Delta_{[i]}}{2}}
\label{eq:Jn1}
\end{align}
We now bound the above terms individually as follows.
\begin{align}
& \sum_{k=1}^{K-1} \sum_{i=K+1-k}^K \prob{c^{[i]}_\alpha - \hat{c}^{[i]}_{{n_k}, \alpha} \geq \frac{\Delta_{[i]}}{2}} 
 \leq \sum_{k=1}^{K-1} \sum_{i=K+1-k}^K \prob{\vert \hat{c}^{[i]}_{{n_k}, \alpha} - c^{[i]}_\alpha \vert \geq \frac{\Delta_{[i]}}{2}} \nonumber \\  
&\qquad\qquad\qquad\overset{(a)}{\leq}\!\sum_{k=1}^{K-1}\!\sum_{i=K+1-k}^K 8 \exp \left( -n(1-\alpha) \min \lbrace \frac{\Delta_{[i]}}{2}, \frac{\Delta_{[i]}^2}{4} \rbrace G_{[i]} \right) \nonumber \\
&\qquad\qquad\qquad \leq\!\sum_{k=1}^{K-1}\!\sum_{i=K+1-k}^K \!8\!\exp \left( -n(1-\alpha) \min \lbrace \frac{\Delta_{[i]}}{2}, \frac{\Delta_{[i]}^2}{4} \rbrace G_{\max} \right), \nonumber \\
&\qquad\qquad\qquad \leq \sum_{k=1}^{K-1} 8k \exp \left( -n(1-\alpha) \min \lbrace \frac{\Delta_{[K+1-k]}}{2}, \frac{\Delta_{[K+1-k]}^2}{4} \rbrace\times G_{\max} \right), \label{eq:Jn2}
\end{align}
where $(a)$ is due to Theorem~\ref{thm:cvar-concentration-sub-exp} and\eqref{eq:cvar-concentration-sub-exp-simplified}, and $G_{\max} = \max_i G_i$. Further, note that
\begin{align*}
n \min \lbrace \frac{\Delta_{[K+1-k]}}{2}, \frac{\Delta_{[K+1-k]}^2}{4} \rbrace \geq \frac{n-K}{H\overline{\log}K},
\end{align*}
where $H$ is as defined in the theorem statement. By substituting the above in~\eqref{eq:Jn2},  we obtain
\begin{align}
\label{eq:dfg1}
&\sum_{k=1}^{K-1} \sum_{i=K+1-k}^K \prob{c^{[i]}_\alpha - \hat{c}^{[i]}_{{n_k}, \alpha} \geq \frac{\Delta_{[i]}}{2}}  \leq  \sum_{k = 1}^{K-1} 8k \exp \left( -\frac{(n-K)(1-\alpha) G_{\max}}{H\overline{\log}K} \right).
\end{align}
Similarly, we can show that
\begin{align}
\label{eq:dfg2}
& \sum_{k=1}^{K-1} \sum_{i=K+1-k}^K \prob{\hat{c}^{i^*}_{{n_k}, \alpha} - c^{i^*}_\alpha - \geq \frac{\Delta_{[i]}}{2}}  \leq  \sum_{k = 1}^{K-1} 8k \exp \left( -\frac{(n-K)(1-\alpha) G_{\max}}{H\overline{\log}K} \right).
\end{align}
The main claim follows by substituting~\eqref{eq:dfg1} and~\eqref{eq:dfg2} in~\eqref{eq:Jn1}.
%, we get
%\begin{align}
%\prob{J_{n} \neq i^*}\!\leq\!4K(K-1)\!\exp\!\left( -\frac{(n-K)(1-\alpha) G_{\max}}{H\overline{\log}K} \right),
%\end{align} 
%which completes the proof.
\end{proof}

%%%%%%%%%%%%%%%%%%%%%%%%%%%%%%%%%%%%%%%5
% Proof: heavy
%%%%%%%%%%%%%%%%%%%%%%%%%%%%%%%%%%%%%%%%
\subsection{Proof of Theorem \ref{thm:cvar-concentration-bounded-moment}}
\begin{proof}
Notice that
 \begin{align}
\hat{c}_{n, \alpha} &= \hat{v}_{n,\alpha} +\frac{1}{n(1-\alpha)} \sum_{i=1}^n \left( X_i - \hat{v}_{n, \alpha} \right)\indic{\hat v_{n,\alpha} \le X_i \le B_i} \nonumber\\
& = v_\alpha + \frac{1}{n(1-\alpha)} \sum_{i=1}^n \left( X_i - v_\alpha \right)\indic{v_\alpha \le X_i \le B_i} + \e_n, \textrm{ where }\label{eq:asd1} 
\end{align}
\begin{align*}
\e_n &= \left( \hat{v}_{n, \alpha} - v_\alpha \right) + 
\frac{1}{n(1-\alpha)} \sum_{i=1}^n  \left( X_i - \hat{v}_{n, \alpha} \right)\left[ \indic{\hat{v}_{n, \alpha} \le X_i \le B_i} - \indic{v_\alpha \le X_i \le B_i} \right]\\
&= \left( \hat{v}_{n, \alpha} - v_\alpha \right) 
+ \frac{1}{n(1-\alpha)} \sum_{i=1}^n  \left( v_{\alpha} - \hat{v}_{n, \alpha} \right)\indic{\hat{v}_{n, \alpha} \le X_i \le B_i}  \\
&\qquad+ \frac{1}{n(1-\alpha)} \sum_{i=1}^n\left( X_i - v_{\alpha} \right)\left[ \indic{\hat{v}_{n, \alpha} \le X_i \le B_i} - \indic{v_\alpha \le X_i \le B_i} \right] \\
&= \left( \hat{v}_{n, \alpha} - v_\alpha \right) 
+ \frac{\left( v_\alpha- \hat{v}_{n, \alpha}   \right)}{(1-\alpha)} \left( \hat F_n(B_i) - \hat F_n( \hat{v}_{n, \alpha}) \right)\\
&\qquad+ \frac{1}{n(1-\alpha)} \sum_{i=1}^n\left( X_i - v_{\alpha} \right)\left[ \indic{\hat{v}_{n, \alpha} \le X_i \le B_i} - \indic{v_\alpha \le X_i \le B_i} \right] 
\end{align*}
Thus,
\begin{align*}
\left|\e_n\right| &\le
\frac{\vert v_\alpha - \hat{v}_{n, \alpha} \vert}{1-\alpha} \vert \alpha - \hat{F}_n (\hat{v}_{n, \alpha}) \vert  + \frac{\vert v_\alpha - \hat{v}_{n, \alpha} \vert}{1-\alpha} \vert \hat{F}_n(v_\alpha) - \hat{F}_n (\hat{v}_{n, \alpha}) \vert\\
& \leq \frac{\vert v_\alpha - \hat{v}_{n, \alpha} \vert}{1-\alpha} \Big[ 2 \vert \hat{F}_n (\hat{v}_{n, \alpha}) - F(v_\alpha) \vert 
+ \vert \hat{F}_n(v_\alpha) - F(v_\alpha) \vert \Big].\stepcounter{equation}\tag{\theequation}\label{eq:notsosimple}
\end{align*}
Using $\vert \hat{F}_{n} (\hat{v}_{n,\alpha}) - F(v_\alpha) \vert \leq 1/n$, we obtain
\begin{align*}
 \prob{\e_n > \epsilon} &\leq   \prob{ \frac{2}{n} \frac{1}{(1-\alpha)} \vert \hat{v}_{n,\alpha} - v_\alpha \vert > \frac{\epsilon}{2} }  + \prob{\frac{1}{1-\alpha} \vert \hat{v}_{n,\alpha} - v_\alpha \vert \vert \hat{F}_n (v_\alpha) - F(v_\alpha) \vert > \frac{\epsilon}{2}}\\
&\le 2 \exp \left(  -nc_1 (1-\alpha)^2\epsilon^2   \right) + 2 \exp \left(  -n(1-\alpha)^2c_2\epsilon^2   \right),
\end{align*}
where the final inequality uses the concentration result in Lemma \ref{prop:var-concentration-bound1} to obtain the first term, while the second term can be arrived at as in the proof of Theorem \ref{thm:cvar-concentration-sub-exp}. In particular, letting $\epsilon' = \frac{(1-\alpha)\epsilon}{2}$, and using \eqref{eq:term2bd}, we have
\begin{align*}
 &\prob{\vert \hat{v}_{n,\alpha} - v_\alpha \vert \vert \hat{F}_n (v_\alpha) - F(v_\alpha) \vert > \epsilon'} \le 2 \exp \left(  -\frac{nc'\epsilon'^2}{4}   \right) ,
\end{align*}
where the final inequality follows by using DKW inequality for the first term, and VaR concentration result from Lemma \ref{prop:var-concentration-bound1} for the second term, together with the fact that $\hat{F}_n (v_\alpha) \le 1$.
Thus,
\begin{align}
 \prob{\e_n > \epsilon} &\leq   4 \exp \left(  -n(1-\alpha)^2c_3\epsilon^2   \right),
\textrm{ or, equivalently, }
 \e_n \le \sqrt{\frac{\log(4/\delta)}{c_3 n}} \textrm{ w.p. }(1-\delta). \label{eq:en-hpb} 
\end{align}
Hence, we have
\begin{align*}
c_{\alpha} - \hat{c}_{n, \alpha}  &= \frac{1}{1 - \alpha} \left[ \E \left[ \left( X - v_{\alpha}\right) \indic{ v_{\alpha} \le X } \right]  - \frac{1}{n} \sum_{i=1}^n \left( X_i - v_{\alpha} \right)\indic{ v_{\alpha} \le X_i \le B_i}\right] + \e_n\\
& = I_1 - I_2 + e_n,
\end{align*}
where $I_1 = \frac{1}{1 - \alpha} \E \left[ X \  \indic{ v_{\alpha} \le X } \right]  - \frac{1}{n(1-\alpha)} \sum_{i=1}^n  X_i\ \indic{ v_{\alpha} \le X_i \le B_i}$, and\\
$I_2 = \frac{1}{1 - \alpha} \E \left[ v_\alpha \  \indic{ v_{\alpha} \le X } \right]  - \frac{1}{n(1-\alpha)} \sum_{i=1}^n  v_\alpha\ \indic{ v_{\alpha} \le X_i \le B_i}$.
%\todo[inline]{Need to handle $v_\alpha$ term in RHS above}
We bound the $I_1$ term, using a technique from \cite{bubeck2013bandits}, as follows: 
\begin{align*}
 &\frac{1}{1 - \alpha} \E \left[ X \  \indic{ v_{\alpha} \le X } \right]  - \frac{1}{n(1-\alpha)} \sum_{i=1}^n  X_i\ \indic{ v_{\alpha} \le X_i \le B_i}\\
 & = \frac{1}{n(1-\alpha)} \left(\sum_{i=1}^n \E \left[X\ \indic{ X > B_i } \right] + 
   \sum_{i=1}^n  \E \left[  X \  \indic{ v_\alpha \le X \le B_i } \right] - X_i \ \indic{ v_{\alpha} \le X_i \le B_i}\right) \stepcounter{equation}\tag{\theequation}\label{eq:b1}
\\
&  \le \frac{1}{n(1-\alpha)} \sum_{i=1}^n \frac{u}{B_i^{p-1}} +  \frac{1}{(1-\alpha)} \sqrt{\frac{2 B_n^{2-p}u \log(1/\delta)}{n}} + \frac{1}{(1-\alpha)} \frac{2 B_n \log(1/\delta)}{3n}, \textrm{ holds w.p. } (1-\delta),
\end{align*}
where we have used the fact that $\E( X^p) \ge B^{p-1} \E \left[X\ \indic{ X > B } \right]$ to handle the first term in \eqref{eq:b1}, and  Bernstein's inequality to bound the second term there.

Along similar lines, the term $I_2$ is bounded as follows:
\begin{align*}
 &\frac{1}{1 - \alpha} \E \left[ v_\alpha \  \indic{ v_{\alpha} \le X } \right]  - \frac{1}{n(1-\alpha)} \sum_{i=1}^n  v_\alpha\ \indic{ v_{\alpha} \le X_i \le B_i}\\
 & = \frac{v_\alpha}{n(1-\alpha)} \sum_{i=1}^n \E \left[\indic{ X > B_i } \right] + 
   \frac{v_\alpha}{n(1-\alpha)} \sum_{i=1}^n \bigg( \E \left[  \indic{ v_\alpha \le X \le B_i } \right] -  \indic{ v_{\alpha} \le X_i \le B_i}\bigg) \stepcounter{equation}\tag{\theequation}\label{eq:b2}
\\
&  \le \frac{1}{n(1-\alpha)} \sum_{i=1}^n \frac{u}{B_i^{p-1}} +  \frac{v_\alpha}{(1-\alpha)} \sqrt{\frac{\log(1/\delta)}{2n}} , \textrm{ holds w.p. } (1-\delta),
\end{align*}
where we have used Hoeffding's inequality, and $B_i^p \ge B_i^{p-1}$for bounding the second term\footnote{Note that for a fixed $\delta,$ we can assume $B_i>1$ for all $i$ by taking a $u$ large enough.}  in \eqref{eq:b2}, while the first term  is bounded using an argument similar to that used in bounding $I_1$ term above.

Using $B_i =  \left(\frac{u i}{\log(1/\delta)}\right)^{1/p}$, we have, w.p. $(1-\delta)$, 
\begin{align*}
 I_1 &\le \frac{4 u^{1/p}}{(1-\alpha)}\left(\frac{\log(1/\delta)}{n} \right)^{1-1/p},  \textrm{ and }
 & I_2 \le \frac{ u^{1/p}}{(1-\alpha)}\left(\frac{\log(1/\delta)}{n} \right)^{1-1/p}  +  \sqrt{\frac{\log(1/\delta)}{c_4 n}}.
\end{align*}
Combining the bound above, with that in \eqref{eq:en-hpb}, we obtain
\begin{align}
c_{\alpha} - \hat{c}_{n, \alpha} &\le 
  \frac{5 u^{1/p}}{(1-\alpha)}\left(\frac{\log(1/\delta)}{n} \right)^{1-1/p} +  \sqrt{\frac{\log(4/\delta)}{c_5 n}}\nonumber\\
  &\le \frac{5 u^{1/p}}{(1-\alpha)}\max\left(\log(4/\delta)^{1-1/p}, \log(4/\delta)^{1/2} \right) \frac{1}{n^{1-1/p}}, \textrm{ for } 1<p\le 2. \label{eq:s123}
\end{align}

If the second moment is bounded, i.e.,  $p=2$, we have
\begin{align*}
\prob{c_{\alpha} - \hat{c}_{n, \alpha} > \epsilon}  &\le 4\exp\left(-  cn (1-\alpha)^2 \epsilon^2\right),
\end{align*}
where $c$ is a distribution-dependent constant. Along similar lines, a concentration bound for the other tail can be obtained.
Thus, we have
\[\prob{\left| \hat{c}_{n, \alpha} - c_{\alpha}\right| > \epsilon}  \le 8\exp\left(-  cn (1-\alpha)^2 \epsilon^2\right).\]

Similarly, from \eqref{eq:s123}, for the case when $p \in (1,2)$, we obtain
\begin{align*}
\prob{ \left| 
\hat c_{n,\alpha} - c_\alpha \right| > \epsilon} \le 8\exp\left(-  c'n (1-\alpha)^{\frac{p}{(p-1)}} \epsilon^{\frac{p}{(p-1)}}\right),
\end{align*} 
where $c'$ is a distribution-dependent constant.

\end{proof}

%%%%%%%%%%%%%%%%%%%%%%%%%%%%%%%%%%%%%%%%%%%%%%%%%%%%%%%%%%%%%%%%%%%%%%%%%%%%%%%
%%%%%%%%%%%%%%%%%%%%%%%%%%%%%%%%%%%%%%%%%%%%%%%%%%%%%%%%%%%%%%%%%%%%%%%%%%%%%%%
%%%%%%%%%%%%%%%%%%%%%%%%%%%%%%%%%%%%%%%%%%%%%%%%%%%%%%%%%%%%%%%%%%%%%%%%%%%%%%%
\section{Concluding  Remarks}
\label{sec:conclusions}
We derived concentration bounds for CVaR estimation, separately considering light-tailed and heavy-tailed distributions.  For light-tailed distributions, our concentration bound uses a classical CVaR estimator based on the empirical distribution.  For the heavy-tailed case, we employ a truncation based CVaR estimator, and derive a concentration result under a mild bounded-moment assumption. Our concentration bound enjoys exponential decay in the sample size even for heavy-tailed random variables. We highlighted the applicability of the CVaR concentration result by considering a risk-aware best bandit arm selection problem. We proposed an adaptation of the successive rejects algorithm to the setting where the goal is to find an arm with the lowest CVaR. Using the CVaR concentration bound, we established error bounds for the proposed algorithm.

\bibliographystyle{amsplain}
\bibliography{references}

\providecommand{\bysame}{\leavevmode\hbox to3em{\hrulefill}\thinspace}
\providecommand{\MR}{\relax\ifhmode\unskip\space\fi MR }
\providecommand{\MRhref}[2]{%
  \href{http://www.ams.org/mathscinet-getitem?mr=#1}{#2}
}
\providecommand{\href}[2]{#2}
\begin{thebibliography}{10}

\bibitem{artzner1999coherent}
Philippe Artzner, Freddy Delbaen, Jean-Marc Eber, and David Heath,
  \emph{{Coherent measures of risk}}, Mathematical finance \textbf{9} (1999),
  no.~3, 203--228.

\bibitem{audibert2010best}
J.~Y. Audibert, S.~Bubeck, and R.~Munos, \emph{Best arm identification in
  multi-armed bandits}, Conference on Learning Theory, 2010, pp.~41--53.

\bibitem{2019arXiv190210709B}
Sanjay~P. {Bhat} and Prashanth~L. {A}, \emph{{Improved Concentration Bounds for
  Conditional Value-at-Risk and Cumulative Prospect Theory using Wasserstein
  distance}}, arXiv e-prints (2019), arXiv:1902.10709.

\bibitem{brown2007large}
David~B Brown, \emph{Large deviations bounds for estimating conditional
  value-at-risk}, Operations Research Letters \textbf{35} (2007), no.~6,
  722--730.

\bibitem{bubeck2013bandits}
S{\'e}bastien Bubeck, Nicolo Cesa-Bianchi, and G{\'a}bor Lugosi, \emph{Bandits
  with heavy tail}, IEEE Transactions on Information Theory \textbf{59} (2013),
  no.~11, 7711--7717.

\bibitem{chatterjee2014practical}
Rupak Chatterjee, \emph{Practical methods of financial engineering and risk
  management: tools for modern financial professionals}, Apress, 2014.

\bibitem{david2016pure}
Yahel David and Nahum Shimkin, \emph{Pure exploration for max-quantile
  bandits}, Joint European Conference on Machine Learning and Knowledge
  Discovery in Databases, Springer, 2016, pp.~556--571.

\bibitem{davidpac}
Yahel David, Bal{\'a}zs Sz{\"o}r{\'e}nyi, Mohammad Ghavamzadeh, Shie Mannor,
  and Nahum Shimkin, \emph{Pac bandits with risk constraints}, International
  Symposium on Artificial Intelligence and Mathematics, 2018.

\bibitem{fournier2015rate}
Nicolas Fournier and Arnaud Guillin, \emph{On the rate of convergence in
  wasserstein distance of the empirical measure}, Probability Theory and
  Related Fields \textbf{162} (2015), no.~3-4, 707--738.

\bibitem{galichet2013exploration}
Nicolas Galichet, Michele Sebag, and Olivier Teytaud, \emph{Exploration vs
  exploitation vs safety: Risk-aware multi-armed bandits}, Asian Conference on
  Machine Learning, 2013, pp.~245--260.

\bibitem{ravi2018cvar}
R.~K. {Kolla}, L.~A. {Prashanth}, S.~P. {Bhat}, and K.~{Jagannathan},
  \emph{{Concentration bounds for empirical conditional value-at-risk: The
  unbounded case}}, ArXiv e-prints (2018).

\bibitem{sani2012risk}
A.~Sani, A.~Lazaric, and R.~Munos, \emph{Risk-aversion in multi-armed bandits},
  Advances in Neural Information Processing Systems, 2012, pp.~3275--3283.

\bibitem{serfling2009approximation}
Robert~J Serfling, \emph{Approximation theorems of mathematical statistics},
  vol. 162, John Wiley \& Sons, 2009.

\bibitem{sun2010asymptotic}
Lihua Sun and L~Jeff Hong, \emph{Asymptotic representations for
  importance-sampling estimators of value-at-risk and conditional
  value-at-risk}, Operations Research Letters \textbf{38} (2010), no.~4,
  246--251.

\bibitem{thomas2019concentration}
Philip Thomas and Erik Learned-Miller, \emph{Concentration inequalities for
  conditional value at risk}, International Conference on Machine Learning,
  2019, pp.~6225--6233.

\bibitem{thompson1933likelihood}
William~R Thompson, \emph{On the likelihood that one unknown probability
  exceeds another in view of the evidence of two samples}, Biometrika
  \textbf{25} (1933), no.~3/4, 285--294.

\bibitem{wainwright2019high}
Martin~J Wainwright, \emph{High-dimensional statistics: A non-asymptotic
  viewpoint}, vol.~48, Cambridge University Press, 2019.

\bibitem{wang2010deviation}
Ying Wang and Fuqing Gao, \emph{Deviation inequalities for an estimator of the
  conditional value-at-risk}, Operations Research Letters \textbf{38} (2010),
  no.~3, 236--239.

\end{thebibliography}

%%%%%%%%%%%%%%%%%%%%%%%%%%%%%%%%%%%%%%%%%%%%%%%%%%%%%%%%%%%%%%%%%%%
%%                                                               %%
%% You have reached the end of your document.                    %%
%%                                                               %%
%%%%%%%%%%%%%%%%%%%%%%%%%%%%%%%%%%%%%%%%%%%%%%%%%%%%%%%%%%%%%%%%%%%

\end{document}